\newcommand{\longversion}[1]{#1}
\newcommand{\shortversion}[1]{}
\documentclass[10pt,a4paper]{article}
\usepackage[hscale=0.7,scale=0.75]{geometry}

\usepackage{graphicx} %
\usepackage{caption} %
\usepackage{algorithm}
\usepackage{algorithmic}
\usepackage{tikz}
\usepackage{newfloat}
\usepackage{listings}
\DeclareCaptionStyle{ruled}{labelfont=normalfont,labelsep=colon,strut=off} %
\floatstyle{ruled}
\newfloat{listing}{tb}{lst}{}
\floatname{listing}{Listing}

\usepackage{tabularx}
\usepackage{amsthm}
\newtheorem{definition}{Definition}
\newtheorem{example}[definition]{Example}
\newtheorem{remark}[definition]{Remark}

\newcommand{\toolname}[1]{\ensuremath{\mathtt{#1}}}

\newcommand{\sharpP}{{\sf \#P}}

\newcommand{\PSPACE}{{\sf PSPACE}}

\newcommand{\PP} {{\sf PP}}

\newcommand{\sharpSAT}{{\#SAT}}

\usepackage{amsmath}
\usepackage{amsthm}
\usepackage{amssymb}
\usepackage{enumerate}
\usepackage{enumitem}
\usepackage{nicefrac}
\usepackage{booktabs}

\newcommand{\lp}{\Pi}                                         %

\usepackage{mathtools}
\usepackage{xspace}

\newcommand{\complexityClassFont}[1]{\ensuremath{\mathsf{#1}}}
\newcommand{\NP}{\text{\complexityClassFont{NP}}\xspace}
\newcommand{\co}{\complexityClassFont{co}}
\newcommand{\Ptime}{\complexityClassFont{P}\xspace}

\usepackage{stmaryrd}
\usepackage[mathscr]{eucal}
\newcommand{\res}[2]{\ensuremath{#1\llbracket#2\rrbracket}}
\newcommand{\task}{\ensuremath{\Pi}}
\newcommand{\vars}{\ensuremath{\mathcal A}}

\newcommand{\operators}{\ensuremath{\mathcal O}}
\newcommand{\init}{\ensuremath{\mathcal I}}
\newcommand{\goal}{\ensuremath{\mathcal G}}
\DeclareMathOperator{\pre}{pre}
\DeclareMathOperator{\eff}{eff}

\newcommand{\plan}{\ensuremath{\pi}}
\newcommand{\len}[1]{\ensuremath{|#1|}}

 \newcommand{\citex}[1]{\citeauthor{#1}~\shortcite{#1}}

\usepackage{rotating}
\usepackage{multirow}
\usepackage{array}
\newcolumntype{H}{>{\setbox0=\hbox\bgroup}c<{\egroup}@{}}
\usepackage{mathtools}
\newcommand{\eqdef}{\coloneqq}

\DeclareMathOperator{\dom}{dom}
\newcommand{\pname}[1]{\textsc{#1}\xspace}
\DeclareMathOperator{\poly}{\textit{poly}}
\newcommand{\Card}[1]{\ensuremath{|#1|}}

\let\phi=\varphi
\let\epsilon=\varepsilon

\newcommand{\prob}{\mathbb{P}}
\DeclareMathOperator{\AP}{Plans}
\newcommand{\PlanCons}{\pname{Bounded-Plan-Exist}}
\newcommand{\PlanCount}{\pname{\#Bounded-Plan}}

\newcommand{\PolyPlanCautious}{\pname{Poly-Cautious-Plan-Exist}}
\newcommand{\PolyPlanBrave}{\pname{Poly-Brave-Plan-Exist}}
\newcommand{\PolyPlanCons}{\pname{Poly-Bounded-Plan-Exist}}
\newcommand{\PolyPlanCount}{\pname{\#Poly-Bounded-Plan}}
\newcommand{\PolyTopK}{\pname{Poly-Bounded-Top-k-Exist}}

\newcommand{\FacetReason}{\protect\pname{FacetReason}}
\newcommand{\ProbReasoning}{\protect\pname{Poly-Probabilistic-Reason}}

\newcommand{\kFacetWeight}{\protect{\pname{Exact-k-Facets}}\xspace}
\newcommand{\kFacetAtLeast}{\protect{\pname{AtLeast-k-Facets}}\xspace}
\newcommand{\kFacetAtMost}{\protect{\pname{AtMost-k-Facets}}\xspace}

\newcommand{\DP}{\ensuremath{\textsc{D}^\textsc{\Ptime}}\xspace}

\newcommand{\CE}{\ensuremath{\textsc{C}^\textsc{P}_{=}}\xspace}

\newcommand{\bcite}[1]{\ensuremath{[}#1\ensuremath{]}}
\newcommand{\set}[0]{\ensuremath{\bigtriangledown}}

\DeclareMathOperator{\BC}{\mathcal{BO}}                    %
\DeclareMathOperator{\CC}{\mathcal{CO}}                    %
\newcommand{\FA}{\mathcal{F}}                 %
\DeclareMathOperator{\Mod}{Mod}
\DeclareMathOperator{\at}{vars}

\newcommand{\fplan}{F^{\text{plan}}_{\leq\ell}}

\newcommand{\Nat}{\ensuremath{\mathbb{N}}}
\newcommand{\SB}{\{}%
\newcommand{\SM}{\mid}%
\newcommand{\SE}{\}}%

\usepackage[affil-it]{authblk}
\usepackage[english]{babel}
\usepackage{microtype}
\usepackage[misc]{ifsym}%
\newcommand{\email}[1]{\texttt{#1}}

\title{Counting and Reasoning with Plans%
  \thanks{%
    This is an author self-archived and extended version of a paper
    that has been accepted for publication at AAAI’25.\\
    \Letter: \email{davidjakob.speck@unibas.ch}, \email{hecher@cril.fr},
    \email{daniel.gnad@liu.se}, \email{johannes.fichte@liu.se},\newline \mbox{~~~~~}\email{augusto.blaascorrea@chch.ox.ac.uk} }
}
\author{
  David Speck%
}
\affil{University of Basel, Switzerland}

\author{
  Markus Hecher%
}
\affil[2]{
  Univ. Artois, CNRS, UMR 8188, CRIL, F-62300 Lens, France
}
\author{
  Daniel Gnad%
}

\author{
  Johannes K. Fichte%
}
\affil{Link{\"o}ping University, Sweden}

\author{
    Augusto B. Corrêa%
  }
\affil{University of Oxford, United Kingdom}

\usepackage{thmtools}
\usepackage{thm-restate}
\usepackage{apptools}
\usepackage{xpatch}
\makeatletter
\xpatchcmd{\thmt@restatable}%
{\csname #2\@xa\endcsname\ifx\@nx#1\@nx\else[{#1}]\fi}%
{\IfAppendix{\csname #2\@xa\endcsname}{\csname #2\@xa\endcsname\ifx\@nx#1\@nx\else[{#1}]\fi}}%
{}{} %
\makeatother

\makeatletter
\longversion{
  \def\leftcite{\@up[}\def\rightcite{\@up]}
  
  \def\cite{\def\citeauthoryear##1##2{\def\@thisauthor{##1}%
               \ifx \@lastauthor \@thisauthor \relax \else##1, \fi ##2}\@icite}
  \def\shortcite{\def\citeauthoryear##1##2{##2}\@icite}
  
  \def\citeauthor{\def\citeauthoryear##1##2{##1}\@nbcite}
  \def\citeyear{\def\citeauthoryear##1##2{##2}\@nbcite}

  \def\@icite{\leavevmode\def\@citeseppen{-1000}%
   \def\@cite##1##2{\leftcite\nobreak\hskip 0in{##1\if@tempswa , ##2\fi}\rightcite}%
   \@ifnextchar [{\@tempswatrue\@citex}{\@tempswafalse\@citex[]}}
  \def\@nbcite{\leavevmode\def\@citeseppen{1000}%
   \def\@cite##1##2{{##1\if@tempswa , ##2\fi}}%
   \@ifnextchar [{\@tempswatrue\@citex}{\@tempswafalse\@citex[]}}

  \def\@citex[#1]#2{%
    \def\@lastauthor{}\def\@citea{}%
    \@cite{\@for\@citeb:=#2\do
      {\@citea\def\@citea{;\penalty\@citeseppen\ }%
       \if@filesw\immediate\write\@auxout{\string\citation{\@citeb}}\fi
       \@ifundefined{b@\@citeb}{\def\@thisauthor{}{\bf ?}\@warning
         {Citation `\@citeb' on page \thepage \space undefined}}%
       {\csname b@\@citeb\endcsname}\let\@lastauthor\@thisauthor}}{#1}}

  \def\@biblabel#1{\def\citeauthoryear##1##2{##1, ##2}\@up{[}#1\@up{]}\hfill}
  
  \def\@up#1{\leavevmode\raise.2ex\hbox{#1}}
}
\makeatother

\usepackage{relsize}
\usepackage{nicefrac}

\usepackage{hyperref}
\usepackage[hyphenbreaks]{breakurl}
\hypersetup{colorlinks,allcolors=black}
\usepackage{xurl}
\hypersetup{breaklinks=true}

\begin{document}

\maketitle

\begin{abstract}
  Classical planning asks for a sequence of operators reaching a given goal.
  While the most common case is to compute a plan, many %
  scenarios require more than that.
  However, quantitative reasoning on the plan space remains mostly unexplored.
  A fundamental problem is to \emph{count plans}, which relates
  to the conditional probability on the plan space. Indeed, qualitative and
  quantitative approaches are well-established in various other areas of
  automated reasoning.

  We present the first study to quantitative and qualitative reasoning on the
  plan space. In particular, we focus on polynomially bounded plans. On the
  theoretical side, we study its complexity, which gives rise
  to rich reasoning modes.
  Since counting is hard in general, we introduce the easier notion of facets,
  which enables understanding the significance of operators.
  On the practical side, we implement quantitative reasoning for
  planning. Thereby, we transform a planning task into a propositional
  formula and use knowledge compilation %
  to count different plans.
  This framework scales well to large plan spaces, while enabling rich
  reasoning capabilities such as learning pruning functions and
  explainable planning.
\end{abstract}

\section{Introduction}

The overarching objective of classical planning is to find a plan, i.e., a sequence of operators, that transforms the current state into a goal state. 
While in some scenarios a single plan is sufficient, in others, it may not be clear which plan is preferable based on the description of the planning task. 
To address this, solvers like top-k or top-quality planners have been developed to enumerate the $k$ shortest plans or all plans up to a certain length bound %
allowing for post hoc consideration of the plan space and selection %
\cite{katz-et-al-icaps2018,katz-sohrabi-aaai2020,speck-et-al-aaai2020,vontschammer-et-al-icaps2022,chakraborti-et-al-aaai2024}. 
Although this paradigm has been successfully applied in practical areas such as malware detection \cite{boddy-et-al-icaps2005} and scenario planning for risk management \cite{sohrabi-et-al-aaai2018}, it remains an indirect method for reasoning about the plan space of a planning task.

Considering fundamental problems in computer science, such as the propositional satisfiability problem (SAT), answer set programming (ASP), and constraint satisfaction problems (CSP), more directed reasoning schemes exist that are anchored around counting. %
The most prominent and canonical counting problem is \sharpSAT{}, also called \emph{model counting}, which asks to compute the number of models of %
a %
formula.
While \sharpSAT{} is considered computationally harder than asking whether a single model exists (SAT), it also allows for automated reasoning about the solution space~\cite{darwiche-jacm2001,darwiche-marquis-jair2002}.
Recent competitions illustrate that, despite high computational complexity, state-of-the-art solvers are effective in practice~\cite{fichte-et-al-jea2021}.
Due %
favorable reasoning power and vast applications, %
counting techniques have been extended to
other fields~\cite{aziz-et-al-aaai2015,fichte-et-al-lpnmr2017,hahn-et-al-rulemlpr2022,eiter-et-al-aij2024}.

In this paper, we bridge the gap between model counting and %
classical planning by introducing a new framework for reasoning and analyzing plan space.
To do so, we consider all plans for a given planning task with polynomially bounded length, consistent with the approach used in top-quality planning \cite{katz-sohrabi-aaai2020}.
\paragraph{Contributions}
Our main contributions are as follows:
\begin{enumerate}
    \item We introduce a \emph{taxonomy of counting and reasoning problems  for classical planning} with polynomially bounded plan lengths and establish the computational complexity of these problems.
    \item We identify a class of reasoning problems on the plan space, called \emph{facet reasoning}, that are as hard as polynomially bounded planning and thus can be solved more efficiently than counting problems.
    \item We present a practical tool, \toolname{Planalyst}, that builds on existing planning and knowledge compilation techniques to answer plan-space reasoning queries and demonstrate its practical feasibility.
\end{enumerate}

\paragraph{In more detail,}on the theoretical side, we formally define a taxonomy of counting and reasoning problems for planning and analyze the computational complexity of these problems.
Among other results, we show that the problem of probabilistic reasoning about the plan space such as determining how many plans contain a given operator is \CE{}-complete, which is considered computationally harder than counting the number of plans, known to be \sharpP-complete \cite{speck-et-al-aaai2020}.
We also introduce the notion of \emph{facet reasoning} in the context of planning, which has %
origins in computational complexity \cite{papadimitriou-yannakakis-stoc1982} and is well studied in %
ASP~\cite{alrabbaa-et-al-rulemlpr2018,fichte-et-al-aaai2022}. %
We show that facet reasoning in planning is \NP-complete, and thus probably much simpler than counting the number of plans. 
This theoretical result is significant because it allows more efficient answers to complex reasoning queries about the plan space, such as identifying which operators can complement a given partial plan and which provide more flexibility for further complementation.

On the practical side, we present a solution to the studied counting and reasoning problems by transforming a planning task into a propositional formula, where satisfying assignments correspond one-to-one to plans, followed by subsequent knowledge compilation into a d-DNNF \cite{darwiche-marquis-jair2002}.
We implement this as a tool called \toolname{Planalyst}, which builds on existing tools from planning \cite{rintanen-ipc2014} and knowledge compilation \cite{lagniez-marquis-ijcai2017,sundermann-et-al-tsem2024} and thus readily allows plan counting and automated reasoning in plan space.
Empirically, we compare \toolname{Planalyst} to state-of-the-art top-quality planners on the computationally challenging problem of counting plans, and show that our tool performs favorably, especially when the plan space is large and reasoning over trillions of plans is critical. 
Finally, by constructing a d-DNNF, our approach not only supports plan counting, but can also answer reasoning questions such as conditional probability, faceted reasoning, and unbiased uniform plan sampling, all through efficient d-DNNF queries.

\subsection*{Related Work}
\citex{darwiche-marquis-jair2002} detailed the theoretical capabilities and
limitations of normal forms in knowledge compilation.
Established propositional knowledge compilers are
\toolname{c2d}\shortversion{~\cite{darwiche-ijcai1999}}\longversion{~\cite{Darwiche04a}}
and \toolname{d4}\shortversion{~\cite{lagniez-marquis-ijcai2017}},
 new developments are
 extensions of
\toolname{SharpSAT\textnormal{-}TD}~\cite{kiesel-eiter-kr2023}.
Incremental and approximate counting has been considered for
ASP~\cite{KabirEverardoShukla22,FichteGagglHecher24}.
In SAT and ASP, advanced enumeration techniques have also been
studied~\cite{masina-et-al-sat2023,spallitta-et-al-arxiv2023,gebser-et-al-cpaior2009,alviano-et-al-aij2023},
which can be beneficial for counting if
the number of solutions is sufficiently low or when (partial) solutions need to be
materialized.
Exact uniform sampling using knowledge compilation has
also been implemented~\cite{lai-et-al-aaai2021}.
Model counting has been applied to probabilistic planning in the
past~\cite{domshlak-hoffmann-jair2007}.
In classical planning and grounding, \citex{correa-et-al-icaps2023}
argued that grounding is infeasible for some domains if the number of
operators in a planning task is too high. Therefore, they manually
employed model counting, but did not develop extended reasoning
techniques or counting tools for planning.
Fine-grained reasoning modes and facets have been studied for
ASP~\cite{alrabbaa-et-al-rulemlpr2018,fichte-et-al-aaai2022,fichte-et-al-ijcai2022,RusovacHecherGebser24,EiterFichteHecher24}
and significance notions based on facets~\cite{boehl-et-al-ecai2023}.

\newcommand{\polyrow}{\multirow{5}{*}{\rotatebox[origin=c]{90}{$\ell \leq \poly(\Pi)$}}}
\begin{table*}[t]
  \centering
  \begin{tabular}{lllll}
    \toprule
    Name              & Given                   & Task                                       & Compl.    & Ref.                                 \\
    \midrule
    \PolyPlanCons     & $\Pi$, $\ell$           & $\pi \in \AP_\ell(\Pi)$                    & \NP-c     & \bcite{1}                            \\
    \PolyPlanBrave    & $\Pi$, $\ell$, $o$      & $\exists \pi \in \AP_\ell(\Pi): o \in \pi$ & \NP-c     & Lem.~\ref{obs:qualitative}           \\
    \PolyPlanCautious & $\Pi$, $\ell$, $o$      & $\forall \pi \in \AP_\ell(\Pi): o \in \pi$ & \co\NP-c  & Lem.~\ref{obs:qualitative}           \\
    \PolyTopK         & $\Pi$, $\ell$           & $\Card{\AP_\ell} \geq k$                   & \PP-h     & \bcite{2}                            \\ %
    \PolyPlanCount    & $\Pi$, $\ell$           & $\Card{\AP_\ell}$                          & \sharpP-c & \bcite{2}                            \\ %
    \ProbReasoning    & $\Pi$, $\ell$, $Q$, $p$ & $\prob_\ell[\Pi,Q]=p$                      & \CE-c     & Thm.~\ref{thm:prob}                  \\
    \midrule
    \FacetReason      & $\Pi$, $\ell$, $o$      & $o \in \FA_\ell(\Pi)$                      & $\NP$-c   & Thm.~\ref{the:complexity:facet_dec}  \\
    \kFacetAtLeast    & $\Pi$, $\ell$, $k$      & $\Card{\FA_\ell(\Pi)} \geq k$              & \NP-c     & Lem.~\ref{thm:atlestfacet}           \\
    \kFacetAtMost     & $\Pi$, $\ell$, $k$      & $\Card{\FA_\ell(\Pi)} \leq k$              & \co\NP-c  & Cor.~\ref{cor:complexity:atmostfacet}\\
    \kFacetWeight     & $\Pi$, $\ell$, $k$      & $\Card{\FA_\ell(\Pi)} = k$                 & $\DP$-c   & Thm.~\ref{thm:exactlyfacet}          \\
    \bottomrule
  \end{tabular}~                                                                                                                                                                       \\[-.5em]
  \caption{\textit{Computational Complexity of Qualitative and Quantitative Reasoning Problems.}
    We let $\Pi$ be a planning task, $\ell \in \mathbb \Nat_0$ with  $\ell \leq \poly(\Pi)$, $o \in \operators$,
    $k \in \Nat_o$, $0 \leq p \leq 1$, and $Q$ a query. %
    \bcite{1}: \protect\cite{bylander-aij1994}, \bcite{2}:
    {\protect\cite{speck-et-al-aaai2020}}.
  }
\end{table*}

\section{Preliminaries}
We assume that the reader is familiar with basics of propositional
logic~\cite{kleine-lettmann-1999} and computational
complexity~~\cite{papadimitriou-1994}.
Below, we follow standard
definitions~\cite{bylander-aij1994,speck-et-al-aaai2020}
to summarize basic notations for planning.

\paragraph{Basics}
For an integer~$i$, we define $[i] \eqdef \SB 0, 1, \ldots, i\SE$.  We
abbreviate the \emph{domain} of a
function~$f: \mathcal{D} \rightarrow \mathcal{R}$ by~$\dom(f)$. By
$f^{-1}: \mathcal{R} \rightarrow \mathcal{D}$ we denote the inverse
function~$f^{-1} \eqdef \SB f(d) \rightarrow d \SM d \in \dom(f)\SE$
of function~$f$, if it exists.
Let $\sigma = \langle s_1, s_2, \ldots, s_\ell \rangle$ be a sequence,
then we write $s \in \sigma$ if $s = s_i$ for
some~$1 \leq i \leq \ell$ and $\set({\sigma})$ the set of elements
that occur in~$\sigma$,~i.e.,
$\set({\sigma}) \eqdef \SB s \SM s \in \sigma \SE$.
For a propositional formula~$F$, we abbreviate by $\at(F)$ the
variables that occur in~$F$ and by $\Mod(F)$ the set of all models
of~$F$ and the number of models by $\#(F)\eqdef \Card{\Mod(F)}$.

\paragraph{Computational Complexity}
We follow standard terminology in computational
complexity~\cite{papadimitriou-1994} and
the Polynomial Hierarchy
(PH)~\cite{stockmeyer-meyer-stoc1973,stockmeyer-tcs1976,wrathall-tcs1976}.
The complexity class $\DP$ captures the (independent) combination of
an $\NP$ and a $\co\NP$
problem, i.e., %
$\DP\eqdef \SB L_1 \cap L_2 \SM L_1\in \NP, L_2 \in
\co\NP\SE$
\cite{papadimitriou-yannakakis-stoc1982}. %
Class $\PP$~\cite{gill-sicomp1977} refers to those decision
problems that can be characterized by a nondeterministic Turing machine,
such that the positive instances are those where at least $1/2$ of the machine's paths
are accepting.
Counting class $\sharpP$ captures counting problems that can be solved by counting the number of accepting
paths of a nondeterministic Turing machine~\cite{valiant-tcs1979}.
Class $\CE$~\cite{fenner-et-al-eccc1999} refers to decision problems that can be characterized
via nondeterministic Turing machines
where positive instances
are those with the same number of accepting and rejecting paths.

\paragraph{Classical Planning}
A \emph{planning task} is a tuple $\task = \langle \vars, \operators, \init,
\goal \rangle$, where $\vars$ is a finite set of propositional \emph{state
variables}.  A \emph{(partial) state} $s$ is a total (partial) mapping~$s: \vars
\to \{0,1\}$.
For a state~$s$ and a partial state~$p$, we write $s \models p$ if $s$
\emph{satisfies} $p$, more formally, $p ^{-1}(0) \subseteq s^{-1}(0)$
and $p ^{-1}(1) \subseteq s^{-1}(1)$.
$\operators$ is a finite set of \emph{operators},
where each operator is a tuple~$o = \langle \pre_o, \eff_o \rangle$ of
partial states, called \emph{preconditions} and \emph{effects}.
An operator~$o \in \operators$ is \emph{applicable} in a state~$s$ if
$s \models \pre_o$.  \emph{Applying} operator~$o$ to state~$s$,
$\res{s}{o}$ for short, yields state~$s'$, where
$s'(a) \eqdef \eff_o(a)$, if $a \in \dom(\eff_o)$ and
$s'(a) \eqdef s(a)$, otherwise.
Finally, $\init$ is the \emph{initial state} of~$\task$ and $\goal$
a partial state called \emph{goal condition}. A state $s_*$ is a \emph{goal
state} if $s_* \models \goal$.
Let $\task$ be a planning task.
A \emph{plan} $\plan = \langle o_0, \dots, o_{n-1}
\rangle$ is a sequence of applicable operators that \emph{generates} a
sequence of states $s_0, \dots, s_n$, where $s_0 = \init$,
$s_n$ is a goal state, and $s_{i+1} = \res{s_i}{o_i}$ for every~$i \in
[n-1]$. Furthermore, we let~$\pi(i) \eqdef
o_i$ and denote by $\len{\pi}$ the \emph{length} of a plan
$\pi$.
We denote the set of all plans
by~$\AP(\Pi)$ and the set of all plans of length at
most~$\ell$
by~$\AP_\ell(\Pi)$ and call it occasionally \emph{plan space} as done
in the literature~\cite{russell-norvig-1995}.

A plan~$\pi$ is \emph{optimal} if there is no plan~$\pi' \in
\AP(\Pi)$ where $\Card{\pi'} <
\Card{\pi}$. The notion naturally extends to bounded-length plans.
Deciding or counting plans is computationally hard. More precisely,
the \PlanCons problem, which asks to decide whether there exists a
plan of length at most~$\ell$, is \PSPACE-complete~\cite{bylander-aij1994}.
The \PlanCount problem, which asks to output the number of plans of
length at
most~$\ell$, remains \PSPACE-complete~\cite{speck-et-al-aaai2020}.
We say that a plan is \emph{polynomially bounded} if we restrict the
length to be polynomial in the instance size,~i.e., the
length~$\ell$ of~$\Pi$ is bounded by $\ell \leq
\|\Pi\|^c$ for some constant~$c$, where $\|\Pi\|$ is the encoding size of $\Pi$.
For a planning problem~$\mathbb{P}$ with
input~$\ell$ that bounds the length of a plan, we abbreviate by
$\textsc{Poly-}\mathbb{P}$ the problem~$\mathbb{P}$ where
$\ell$ is polynomially bounded.
Then, the complexity drops.
\PolyPlanCons is \NP-complete~\cite{bylander-aij1994} and
\PolyPlanCount is \sharpP-complete, and the decision problem~\PolyTopK
is \PP-hard, which asks to decide, given in addition an
integer~$k$, whether there are at
least~$k$ different plans of length up to~$\ell$~\cite{speck-et-al-aaai2020}.

\begin{figure}[t]
  \centering
  \begin{tikzpicture}[node distance={22mm}, main/.style = {draw, circle}]
    \node[main] (0) {$s_0$};
    \node[main] (1) [right of=0] {$s_1$};
    \node[main, yshift=5mm] (2) [above right of=1] {$s_2$};
    \node[main, yshift=-5mm] (3) [below right of=2]  {$s_3$};
    \node[main, double] (goal) [right of=3] {$s_*$};
    \node[main, yshift=3mm] (bot) [above right of=0] {$s_\bot$};
    \draw[->] (0,0.85) -- (0);
    \draw[->, thick] (0) -- node[midway, above] {wake-up} (1);
    \draw[->, thick] (1) -- node[midway, above, sloped] {get-ready} (2);
    \draw[->, thick] (1) -- node[midway, above, sloped] {go-to-AAAI} (3);
    \draw[->, thick] (2) -- node[midway, above, sloped] {go-to-AAAI} (3);
    \draw[->, thick] (3) -- node[midway, above, sloped] {give-talk} (goal);
    \draw[->, thick] (0) -- node[midway, above, sloped] {sleep} (bot);
  \end{tikzpicture}
\caption{State space of our running example task $\Pi_1$. The initial state is denoted by
  $s_0$; the goal state is denoted by $s_*$.}
\label{fig:running-example}
\end{figure}
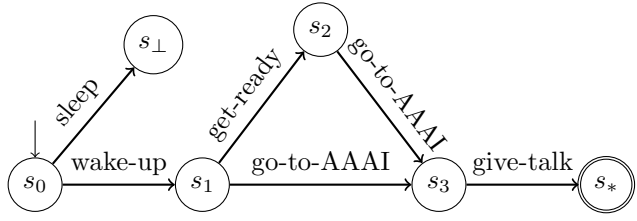

\begin{example}[Running Example]\label{ex:running}
  Consider a planning task~$\Pi_1$ consisting of a scenario with a slightly chaotic researcher, who has to wake up and give a
  talk at AAAI. Depending on how late they are, %
  they can go straight to
  the talk without any preparation. %
  However, they could also spend time getting ready.
  Less pleasant to the audience, they could also continue sleeping and
  not give the talk at all.
  Figure~\ref{fig:running-example} illustrates the  state space. The initial state is~$s_0$, and the single
goal state is~$s_*$. The labels in each edge identify the operator being
applied. We can easily identify two plans: %
\begin{enumerate}[label=(\roman*), leftmargin=1.75em]
\item \textnormal{wake-up; get-ready; go-to-AAAI; give-talk}.
\item \textnormal{wake-up; go-to-AAAI; give-talk}.
\end{enumerate}
Plan~\textit{(i)} has length~$4$, while Plan~\textit{(ii)} has length~$3$. Observe that action
\textnormal{sleep} does not appear in any plan.
\end{example}

\paragraph{Landmarks}
A \emph{fact landmark} is a state variable that occurs in every
plan~\cite{porteous-et-al-ecp2001}.
An \emph{operator landmark} is an operator that occurs in every plan
\cite{richter-et-al-aaai2008,karpas-domshlak-ijcai2009}. We can extend
these notions to \emph{bounded landmarks} where we assume bounded length~$\ell$.

\begin{example}%
  \label{ex:landmarks}
  Consider planning task~$\Pi_1$ from Example~\ref{ex:running}. We
  observe that \textnormal{wake-up}, \textnormal{go-to-AAAI}, and \textnormal{give-talk} are
  operator landmarks.
\end{example}

\newcommand{\sasp}{SAS$^+$}
\newcommand{\PPP}{\mathcal{P}}
\newcommand{\QQQ}{\mathcal{Q}}
\newcommand{\union}{\cup}

\paragraph{Planning as Satisfiability (SAT)}
Let $\task = \langle \vars, \operators, \init, \goal \rangle$ be a
planning task and $\ell>0$ an integer to bound the length of a
potential plan. We can employ a standard technique to encode finding a
plan into a propositional formula and ask for its satisfiability
(SAT)~\cite{kautz-selman-ecai1992,rintanen-aij2012}.
In more detail, we can construct a formula~$\fplan[\Pi]$ whose
models are in one-to-one correspondence with the $\ell$-bounded plans of \task.
For space reasons, we present only the core idea.
The variables are as follows:
$\at(\fplan)=\SB a^i \SM a \in \vars, i \in [\ell] \SE \SE \cup \SB
o^i \SM o \in \operators, i \in [\ell]\SE$.
Variable~$a^i$ indicates the value of state variable~$a$ at the $i$-th
step of the plan. Hence, if $M \in \Mod(\fplan[\Pi])$ and
$a^\ell \in M$, then state variable~$a$ has value~$1$ after
applying operators~$o^0,\ldots,o^{\ell-1}$ to the initial state.
We assume \emph{sequential encodings}, where the following constraints hold.
\begin{enumerate}
\item a set of clauses encoding the value of each state variable at the initial state;
\item a set of clauses encoding the value of each state variable in the goal condition;
\item a set of clauses guaranteeing that no two operators are chosen at the same
  step; and
\item a set of clauses guaranteeing the consistency of state variables
  after an
  operator is applied. %
  If $o^i$ is true and the effect of operator~$o$ makes $a$ true, then
  $a^{i+1}$ must be true.
\end{enumerate}
Since plans might be shorter than~$\ell$,
we move ``unused'' steps to the end using
the formula
$\bigwedge_{i \in [\ell]} (\bigwedge_{o \in \operators} \neg o^i
\rightarrow \bigwedge_{o \in \operators} \neg o^{i+1})$,
which encodes that if no operator was assigned at step~$i$, then no
operator can be assigned at step~$i+1$. Thereby, we obtain a one-to-one
mapping between models of~$\fplan[\Pi]$ and $l$-bounded plans for the task.

\newcommand{\assume}[2]{\ensuremath{#1}[#2]}
\newcommand{\val}{\ensuremath{\mathit{val}}}
\newcommand{\cg}[1][\lp]{\ensuremath{\mathcal{G}(#1)}}         %
\newcommand{\cga}[1][\ass]{\mathcal{G}_{\lp}^{#1}} %
\newcommand{\sddcompl}{\Phi_{\lp}}
\newcommand{\sddcomplL}{\Phi_{\assume{\lp}{L}}}

\section{From Qualitative to Quantitative Reasoning}
Classical planning aims at finding one plan or enumerating certain
plans.
But what if we want plans that contain a
certain operator, or to count the number of possible plans given certain
assumptions, or if we want to identify the frequency of an operator among
all possible plans?
Currently, there is no unified reasoning tool to deal with these types of
questions. We introduce more detailed qualitative and quantitative
reasoning modes for planning and analyze its %
complexity.
We start with two extreme reasoning modes that consider whether an
operator is part of some or all plans.
\begin{definition}
Let $\Pi = \langle \vars, \operators, \init, \goal \rangle$ be a
planning task, $o \in \operators$ an operator, and $\ell$ an integer.
We define the
\begin{itemize}
  \item \emph{brave} operator by
$\BC_\ell(\Pi) \eqdef \bigcup_{\pi \in \AP_\ell(\Pi)} \set(\pi)$ and
\item \emph{cautious} operator
  by~$\CC_\ell(\Pi) \eqdef \bigcap_{\pi \in \AP_\ell(\Pi)} \set(\pi)$.
\end{itemize}
The problem \PolyPlanBrave asks to decide whether
$o \in \BC_\ell(\Pi)$.
The problem \PolyPlanCautious asks to decide whether
$o \in \CC_\ell(\Pi)$.
\end{definition}
Note that we use $\set(\cdot)$ to convert sequences into sets, as we
aim only for %
an operator occurring at any time-point.

\begin{remark}\label{rem:landmarks}
  Our definition of cautious operators is similar
  to~\emph{operator
    landmarks}~\cite{zhu-givan-icaps2003dc},
  but for plans with up to a given bounded length.
\end{remark}

\begin{example}%
  Consider task~$\Pi_1$ from Example~\ref{ex:running} and Plans~(i)
  and (ii). Furthermore, let~$\ell=4$.
  Then, the brave and cautious operators of our task are the
  following:
  \begin{align*}
    \BC_\ell(\Pi_1) &= \{\textnormal{wake-up, get-ready, go-to-AAAI, give-talk}\},\\
    \CC_\ell(\Pi_1) &= \{\textnormal{wake-up, go-to-AAAI, give-talk}\}.
  \end{align*}
  Operator~$\textnormal{get-ready}$ is brave but not cautious, as it
  appears in Plan~(i) but not in Plan~\textit{(ii)}. Operator
  \textnormal{sleep} is neither brave nor cautious, as it does not
  appear in any plan.
\end{example}

\begin{restatable}[$\star$\protect\footnote{We prove statements marked by~``$\star$'' in the appendix.}]{lemma}{obsqualitative}\label{obs:qualitative}
  \label{}
  The problem~\PolyPlanBrave is \NP-complete and the
  problem~\PolyPlanCautious is \co\NP-complete.
\end{restatable}

To find brave operators in practice, we can employ a standard
SAT~\cite{audemard-simon-ijait2018} or ASP
solver~\cite{gebser-et-al-lpnmr2011,gebser-et-al-arxiv2014,alviano-et-al-lpnmr2015}.
For cautious operators, we can employ a dedicated backbone
solver~\cite{biere-et-al-sat2023} or again ASP solvers.

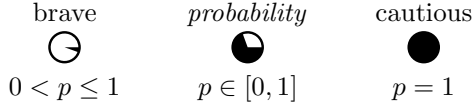
\begin{figure}
\centering
\begin{tabular}{c@{}c@{\hspace{2.5em}}c@{}c@{\hspace{2.5em}}c}
  & brave              & %
  & \emph{probability} & cautious      \\
  &
    \begin{tikzpicture}[every node/.style={circle,inner sep=1ex},scale=1.2]
      \node[fill=black] (as) at (0,0) [label={[text
        height=1.5ex,yshift=-2em]above:{}}] {};%
      \fill[white] (0,0) -- (340:1ex) arc (340:0:1ex) -- cycle;
    \end{tikzpicture}            &
  &
    \begin{tikzpicture}[every node/.style={circle,inner sep=1ex},scale=1.2]
      \node[fill=black] (as) at (0,0) [label={[text
        height=1.5ex,yshift=-2em]above:{}}] {};%
      \fill[white] (0,0) -- (110:1ex) arc (110:0:1ex) -- cycle;
    \end{tikzpicture}            &
                                   \begin{tikzpicture}[every node/.style={circle,inner sep=1ex},scale=1.2]
                                     \node[fill=black] (as) at (0,0) [label={[text
                                       height=1.5ex,yshift=-2em]above:{}}] {};%
                                   \end{tikzpicture}\\[-.75em]%
  & $0 < p \leq 1$    & %
                       & $p \in[0,1]$       & $p=1$
\end{tabular}~\\[-.5em]
\caption{Quantitative reasoning is a fine-grained reasoning mode
  between brave and cautious reasoning. It asks whether a literal
  matches~$\geq p\cdot 100\%$ of the plans for planning task~$\Pi$.}~\\[-1.38em]
\label{fig:torte}
\end{figure}

\subsection{Probability Reasoning}
Both problems~\PolyPlanBrave and \PolyPlanCautious give rise to
extreme reasoning modes on plans. Cautious reasoning is quite strict
and so unlikely to hold in general. %
Brave reasoning is too general and permissive, and thus quite weak in
practice.
Figure~\ref{fig:torte} illustrates the two reasoning modes and a more
fine-grained mode, which we introduce below. This new mode asks
whether the conditional probability of an operator is above a given
threshold.
It generalizes the known~\PolyTopK planning problem, which only asks
whether at least~$k$ plans exists. The crucial ingredient is counting
the number of possible plans and relating them to the number of
possible plans which contain a given operator.
More formally:
Let~$\Pi=\langle \vars, \operators, \init, \goal \rangle$ be a
planning task, $o$ be an operator.
We abbreviate the set of all plans of~$\Pi$ containing~$o$ by
$\AP_\ell(\Pi,o) \eqdef \SB \pi \SM \pi \in \AP_\ell(\Pi),o \in \pi
\SE$.
Then, we define the \emph{conditional probability} of~$o$ in plans
of~$\Pi$
by
\begin{align*}
\prob_\ell[\Pi,
  o]\eqdef\frac{\Card{\AP_\ell(\Pi,o)}}{\max(1,\Card{\AP_\ell(\Pi)})}.
\end{align*}
Note that the usage of~$\max$ prevents division by zero in case
of no possible plan.
Analogously, we can talk about operator~$o$ in position~$i$ by
replacing $o \in \pi$ with $o = \pi(i)$.
With the help of conditional probability, we can define a fine-grained reasoning mode. %
\newcommand{\query}{Q}

To be more flexible, we define \emph{a query~$\query$} as a
propositional formula in conjunctive normal form (CNF) and assume its
meaning as expected.
We let~$\query$ contain variables corresponding to the set~$\vars$ of
state variables, the set~$\operators$ of operators, as well as
of states and operators in position~$i$ (similar to $\mathsmaller\fplan$).
Let $\pi\in\AP_\ell(\Pi)$ be a plan with
$\plan = \langle o_0, \dots, o_{n-1} \rangle$ that generates
sequence~$s_0, \dots, s_n$.
$\pi$ \emph{satisfies} a variable~$v\in \vars$ if there is
some~$i \in [\ell]$ such that $s_i(v)=1$;
\emph{satisfies} an operator~$o \in \operators$ if there is
some~$i \in [\ell]$ such that $\pi(i) = o$, analogously for fixed
time-points $i$.
Then, $\pi$ satisfies $\neg v$ %
if $\pi$ does not satisfy $v$. %
A plan~$\pi$ satisfies a clause~$C$ in $\query$, if $\pi$ satisfies
one of its literals;
$\pi$ satisfies~$\query$, denoted $\pi\models Q$, if it satisfies
every clause in~$Q$.
We define $\AP_\ell(\Pi,Q) \eqdef \SB \pi \SM \pi \in \AP_\ell(\Pi), \pi\models Q
\SE$.

\begin{definition}[Probability Reasoning]
  Let~$\Pi=\langle \vars, \operators, \init, \goal \rangle$ be a
  planning task, $\ell>0$ be an integer, $\query$ be a query,
  and~$0 {\leq} p {\leq} 1$ with $p\in \mathbb{Q}$. Then,
  \emph{probability reasoning} on~$\query$ asks
  if~$\prob_\ell[\Pi, \query]= p$, where
  \begin{align*}
    \prob_\ell[\Pi,Q]\eqdef\frac{\Card{\AP_\ell(\Pi,Q)}}{\max(1,\Card{\AP_\ell(\Pi)})}.
  \end{align*}
\end{definition}

\begin{example}[Probability Reasoning]
  Again, consider planning task~$\Pi_1$ from Example~\ref{ex:running}
  and let $\ell=4$.
  Take the following probability reasoning queries:
  (i) $\prob_\ell[\Pi_1, \textnormal{wake-up}] = 1$,
  (ii) $\prob_\ell[\Pi_1, \textnormal{get-ready}] = 0.5$, and
  (iii) $\prob_\ell[\Pi_1, \textnormal{sleep}] = 0$.
  Reasoning (i) %
  illustrates that the researcher must always use operator
  \textnormal{wake-up} to reach a goal;
  (ii) indicates that \textnormal{get-ready} occurs in half of the
  plans;
  (iii)~allows us to conclude that no plan uses operator
  \textnormal{sleep}.
  More complex queries might ask for the probability of a plan
  containing both \textnormal{wake-up} and \textnormal{sleep}, or at
  least one of them:
  \begin{align*}
    &\prob_\ell[\Pi_1, \textnormal{wake-up} \land \textnormal{sleep}] = 0,\\
    &\prob_\ell[\Pi_1, \textnormal{wake-up} \lor \textnormal{sleep}] = 1.
  \end{align*}
\end{example}

Probability reasoning can be achieved by counting twice, which
is computationally hard. %
In more detail, we obtain:

\begin{restatable}[$\star$]{theorem}{thmprob}\label{thm:prob}
  The problem \ProbReasoning %
  is $\CE$-complete.
\end{restatable}

\section{Faceted Reasoning}
Above, we %
introduced three different reasoning modes, namely brave, probability,
cautious reasoning. Unfortunately the most precise reasoning mode
---the probability mode--- is the computational most expensive one and
requires to count plans.
Therefore, we turn our attention to reasoning that is less hard than
probabilistic reasoning and allows us still to filter plans and
quantify uncertainty among plans. %
We call this reasoning \emph{faceted reasoning} following terminology
from combinatorics~\cite{papadimitriou-yannakakis-stoc1982} and
ASP~\cite{alrabbaa-et-al-rulemlpr2018}.
At the heart of these tasks is a combination of brave and cautious
reasoning.
These are particularly useful if we want to develop plans
gradually/incrementally to see at a given time point, which operators
are still possible or have the biggest effect.
We focus on operators
that %
belong to some (brave) but not to all plans (cautious).

More formally, for a planning task~$\Pi$ and an integer~$\ell$, we
let~$\FA^+_\ell(\Pi) \eqdef \BC_\ell(\Pi) \setminus \CC_\ell(\Pi)$ and
call the elements of~$\FA^+_\ell(\Pi)$ \emph{inclusive facets}.
In addition, we distinguish \emph{excluding facets} $\FA^-_\ell(\Pi)$, which indicate
that operators are not part of a plan. More formally, we let
$\FA^-_\ell \eqdef \SB \neg o \SM o \in \FA^+(\Pi)\SE$ and define the
set~$\FA_\ell(\Pi)$ of all facets by
$\FA_\ell(\Pi) \eqdef \FA^+_\ell(\Pi) \cup \FA^-_\ell(\Pi)$.
Interestingly, a facet~$p \in \{o,\neg o\}$ is directly related to
\emph{uncertainty}, since the operator~$o$ can either be included in
or be excluded from a plan.
When we \emph{enforce} that a \emph{facet}~$p \in \SB o, \neg o\SE$ is
present in a plan, which we abbreviate by~$\Pi[p]$, we immediately
reduce uncertainty on operators among the plans.
Based on this understanding, we define the notion
of~\emph{significance} for a planning
task~$\Pi=\langle \vars, \operators, \init, \goal \rangle$ and an
operator~$o \in \operators$: %
\newcommand{\signf}{\ensuremath{\mathbb{S}}}
\begin{align*}
  \signf_\ell(\Pi, o) \eqdef
  \frac{\Card{\FA_\ell(\Pi)} -  \Card{\FA_\ell(\Pi[o])}}{\Card{\FA_\ell(\Pi)}}.
\end{align*}
\noindent Note that the notion of significance is particularly
interesting when we already have a
prefix~$\omega_k = \langle o_0, \dots, o_k \rangle$ and are interested
in plans that complete the prefix. Here, facets can assist in
understanding which operator is the most significant for the next
step or some step in the future.
Furthermore, we can include state variables into significance
notations without effect on the complexity.
We omit these cases from the presentation due to space constraints and
readability of our introduced notion.

\subsection{Computational Aspects of Facets}
Next, we study the computational complexity for problems related to
facets.
We limit ourselves to including facets, assume the case where an
operator occurs in some step, and we omit prefixes in the following.
These restrictions have only a negligible effect on the complexity.
We start with a natural reasoning problem: %
The \FacetReason problem asks, given a planning task~$\Pi$ and an
operator $o \in \operators$, to decide whether $o \in \FA(\Pi)$.
We start with a lower and upper bound on the $\FacetReason$ problem.

\newcommand{\ft}[1]{\ensuremath{F_{\text{#1}}}}
\begin{restatable}[$\star$]{theorem}{thecomplexityfacetdec}\label{the:complexity:facet_dec}
  Let $\Pi$ be a planning task and $o \in \operators$. The problem
  \FacetReason is $\NP$-complete.
\end{restatable}

Next, we look into counting facets and first observe that the number of
facets is bound by $0 \leq \Card{\FA(\Pi)} \leq \Card{\operators}$ for
a planning task~$\Pi$.
Therefore, we consider a parameterized version by taking a bound~$k$
on the number of facets as input.
Then, the problem \kFacetWeight asks, given a planning task~$\Pi$ and
an integer~$k$, to decide whether $\Card{\FA(\Pi)} = k$.
Before, we look into upper and lower bounds by the
problems~\kFacetAtLeast and \kFacetAtMost, which ask whether
$\Card{\FA(\Pi)} \geq k$ and $\Card{\FA(\Pi)} \leq k$, respectively.

\begin{restatable}[$\star$]{lemma}{thmatlestfacet}\label{thm:atlestfacet}
  Let $\Pi$ be a planning task, and $\ell\in\Nat$, $k \in \Nat_0$ be
  integers. \kFacetAtLeast is $\NP$-complete.
\end{restatable}

\begin{restatable}[$\star$]{corollary}{corcomplexityatmostfacet}\label{cor:complexity:atmostfacet}
  Let $\Pi$ be a planning task, $\ell\in\Nat$, $k \in \Nat_0$. Then, the problem
  \kFacetAtMost %
  is  $\co\NP$-complete.
\end{restatable}

Both results together yield $\DP$-completeness.

\begin{restatable}[$\star$]{theorem}{thmexactlyfacet}\label{thm:exactlyfacet}
  Let $\Pi$ be a program, and $\ell\in\Nat$, $k \in \Nat_0$ be
  integers. The problem
  \kFacetWeight %
  is $\DP$-complete.
\end{restatable}

\section{Discussion: Applications of Plan Reasoning}

Our new reasoning modes %
offer a rich framework to query the solution
space of planning tasks. In Remark~\ref{rem:landmarks}, %
we discussed the connection between
landmarks and cautious reasoning. Similarly, with brave and cautious reasoning
it is easy to answer questions such as ``does operator $o$ appear on any
plan?'', or ``does partial state $p$ occur on any trajectory?''

The expressiveness of the queries goes way beyond %
and can be
leveraged in many existing planning techniques. For example, determining the set
of operators that are always or never part of a plan is important for learning
pruning functions \cite{gnad-et-al-aaai2019}. We can generalize these more
global queries to reason about operators being only (never) applied in states
that satisfy certain conditions, which is essential for learning policies
\cite{krajnansky-et-al-ecai2014,bonet-geffner-ijcai2015}. Furthermore, brave and
cautious reasoning can be helpful for model debugging, offering a convenient
tool to find out if an operator expected to occur in a plan does in fact never
appear \cite{lin-et-al-aaai2023,gragera-et-al-icaps2023}. In over-subscription
planning \cite{smith-icaps2004}, we can determine the achievability of soft
goals or compute the achievable maximum set of soft goals by answering multiple
queries. This can be utilized in explainable planning, providing reasons for the
absence of solutions that achieve the desired set of soft goals
\cite{eifler-et-al-aaai2020,krarup-et-al-jair2021}. We can even generalize the notion of soft goals to
desired state atoms that are achieved \emph{along} a plan, but which might no
longer hold in the goal.

With faceted reasoning, we are able to answer plan-space queries without actually
counting the number of solutions. This reduces the complexity of answering
queries to \NP-completeness, making reasoning much more practically usable. What
makes facet reasoning particularly interesting is that it allows to efficiently
answer conditional queries, such as ``if I want operator $o$ to occur at step
$k$, how much choice is left for the remaining operators?''.
Similar to previous
work in ASP, facet reasoning allows for an interactive querying mode in which
users can gain insights about the particular solution space of a planning task
\cite{fichte-et-al-aaai2022}. For tasks with a large
set of plans that cannot possibly be navigated manually, facets offer the
possibility to systematically navigate the solution space, narrowing down the
set of plans by committing to desired operators. The \toolname{Planalyst} tool,
which we describe in more detail in the next section, enables this form of
interactive exploration in the context of classical planning.

\section{Empirical Evaluation}
We implemented our %
reasoning framework for classical planning as a tool called \toolname{Planalyst}.
Therefore, we transform planning tasks into SAT formulas based on the \toolname{Madagascar} planner~\cite{rintanen-ipc2011,rintanen-ipc2014}.
To efficiently carry out counting, we use \toolname{d4}~\cite{lagniez-marquis-ijcai2017,audemard-et-al-sat2022}, which compiles (potentially large) formulas into a specialized normal form called \emph{d-DNNF} \cite{darwiche-marquis-jair2002}, enabling fast reasoning.
Finally, we reason %
over the plan space %
via counting queries using the \toolname{ddnnife} reasoner~\cite{sundermann-et-al-tsem2024}, which works in poly-time on d-DNNFs.

\subsection{Experimental Setup}
We focus on solving \PlanCount{}, i.e., counting the number of plans, %
which is the computationally hardest problem studied above. %
This %
allows us to address all reasoning questions discussed, including computing conditional probabilities.
For each task of the benchmark set, we defined an upper bound by collecting known bounds from \toolname{planning.domains} \cite{muise-icaps2016systemdemos} and running winning planners from the most recent International Planning Competitions (IPC) \cite{taitler-et-al-aimag2024}.
In the experiments, we count plans of length up to a multiplicative factor~$c \in \{1.0,1.1,1.2,1.3,1.4,1.5\}$ of the collected upper bounds.
We consider two different configurations for our approach: \toolname{Count}, which only counts the number of plans, and \toolname{Enum}, which additionally enumerates all plans, resulting in a novel top-quality planner for classical planning with unit operator costs.
For comparison, we have chosen two top-quality planners, \toolname{K^*} \cite{katz-et-al-icaps2018} and \toolname{SymK} \cite{speck-et-al-aaai2020}, both of which can be readily used to count the number of plans as they enumerate them, and both of which are considered to scale well to large numbers of plans.
We ran both baseline planners in their recommended configurations\footnote{We disabled a default optimization that removes operators causally irrelevant  to the goal, as it prunes valid plans.}: \toolname{K^*}, which implements orbit-space search \cite{katz-lee-ijcai2023} with the landmark-cut heuristic \cite{helmert-domshlak-icaps2009}, and \toolname{SymK}, which implements a variant of bidirectional symbolic search \cite{torralba-et-al-aij2017}. For enumeration approaches (\toolname{K^*}, \toolname{SymK}, \toolname{Enum}), we let these solvers enumerate the plans only internally to avoid writing billions (or more) of plans to the disk.
All experiments ran on Intel Xeon Silver 4114 processors running at 2.2 GHz.
We used a time limit of 30 minutes and a memory limit of 6 GiB per task.
Our benchmarks include all optimal planning domains from IPCs 1998-2023 with unit operator costs and without conditional effects or axioms.
Source code, benchmarks, and data are available online \cite{speck-et-al-zenodo2024}.

\begin{table}[t]

    \centering
        \newcommand{\rot}[1]{\rotatebox[origin=c]{90}{#1}}
\begin{tabular}{ccccc|ccc}
                               & \multicolumn{4}{c}{Coverage} & \multicolumn{3}{c}{\#Plans}                                                                                                                                                       \\
    \cmidrule(lr){2-5} \cmidrule(lr){6-8}
    \shortstack{Length\\Bound} & \rot{\toolname{K^*}}       & \rot{\toolname{SymK}}       & \rot{\textnormal{\ttfamily\bfseries Enum}} & \rot{\textnormal{\ttfamily\bfseries Count}} & \rot{Max}                 & \rot{Mean}                & \rot{Median}             \\
    \midrule
    $\times$ 1.0               & \textbf{351}                 & 309                         & 253                            & 335                             & $>$10\textsuperscript{15} & $>$10\textsuperscript{13} & $>$10\textsuperscript{2} \\
    $\times$ 1.1               & 289                          & 231                         & 182                            & \textbf{300}                    & $>$10\textsuperscript{15} & $>$10\textsuperscript{13} & $>$10\textsuperscript{4} \\
    $\times$ 1.2               & 212                          & 173                         & 130                            & \textbf{251}                    & $>$10\textsuperscript{15} & $>$10\textsuperscript{13} & $>$10\textsuperscript{5} \\
    $\times$ 1.3               & 177                          & 135                         & 101                            & \textbf{210}                    & $>$10\textsuperscript{18} & $>$10\textsuperscript{15} & $>$10\textsuperscript{5} \\
    $\times$ 1.4               & 142                          & 112                         & \phantom{0}77                  & \textbf{189}                    & $>$10\textsuperscript{21} & $>$10\textsuperscript{18} & $>$10\textsuperscript{6} \\
    $\times$ 1.5               & 112                          & \phantom{0}91               & \phantom{0}61                  & \textbf{170}                    & $>$10\textsuperscript{21} & $>$10\textsuperscript{18} & $>$10\textsuperscript{6} \\
    \bottomrule
\end{tabular}

    \caption{
        (Left): Coverage, i.e., the number of tasks where the number of plans within a multiplicative factor of a length bound was found by \toolname{K^*}, \toolname{SymK}, and our SAT-based approaches, \toolname{Count} and \toolname{Enum}. \toolname{Count} only counts plans, while \toolname{Enum} additionally enumerates them.
        (Right): Statistics on the number of plans in the benchmark set, considering the length bound determined by the four solvers.
      }
    \label{tab:overall_coverage}
\end{table}

\subsection{Overall Performance} %
Table~\ref{tab:overall_coverage} (left) compares the coverage,~i.e., the number of tasks for which different approaches can determine the number of plans, for different multiplicative length bounds. \toolname{K^*} has the best coverage for a length bound of $1.0$.
Our enumeration approach, \toolname{Enum}, ranks overall last, although being able to solve a notable number of tasks by first creating a d-DNNF, followed by a subsequent enumeration query for all models, and finally mapping them to actual plans.
For the 1.0 bound, our counting approach \toolname{Count} performs worse than \toolname{K^*}, but has better coverage than the \toolname{SymK} planner.
When considering higher length bounds, the counting approach, \toolname{Count}, has the highest coverage. The gap between \toolname{Count} and the other approaches gets larger as the length bound increases.
This can be explained by the increasing number of plans, %
see Table~\ref{tab:overall_coverage} (right), where enumeration becomes less feasible due to the large plan space.
This highlights the usefulness of our approach for sampling or reasoning in tasks with huge plan spaces. For example, in scenarios where end-users want to understand the plan space, enumerating over a sextillion ($10^{21}$) different plans is infeasible, but counting them (and using the related reasoning) is possible. Moreover, a decent performance with larger bounds gives us more flexibility for problems where a good bound is not easily available but an over-approximation is,~e.g., using a non-admissible heuristic to come up with a bound.

\begin{table}[!t]

    \centering
    \small
    \setlength{\tabcolsep}{3.5pt}
        \newcommand{\rot}[1]{\rotatebox[origin=c]{90}{#1}}
\begin{tabular}{lrrrr|rrrr}
                    & \multicolumn{4}{c}{Bound: $\times$1} & \multicolumn{4}{c}{Bound: $\times$1.5}                                                                                                                                                                                        \\
    \cmidrule(lr){2-5} \cmidrule(lr){6-9}
    Domains         & \rot{\toolname{K^*}}               & \rot{\toolname{SymK}}                  & \rot{\textnormal{\ttfamily\bfseries Enum}} & \rot{\textnormal{\ttfamily\bfseries Count}} & \rot{\toolname{K^*}} & \rot{\toolname{SymK}} & \rot{\textnormal{\ttfamily\bfseries Enum}} & \rot{\textnormal{\ttfamily\bfseries Count}} \\
    \midrule
    airport (49)    & 7                                    & 7                                      & 7                              & \textbf{11}                     & 7                      & 7                     & 6                              & \textbf{11}                     \\
    barman (14)     & \textbf{3}                           & 0                                      & 0                              & 0                               & \textbf{0}             & \textbf{0}            & \textbf{0}                     & \textbf{0}                      \\
    blocks (35)     & 28                                   & 31                                     & 29                             & \textbf{33}                     & 9                      & 8                     & 7                              & \textbf{15}                     \\
    childsnack (20) & \textbf{0}                           & \textbf{0}                             & \textbf{0}                     & \textbf{0}                      & \textbf{0}             & \textbf{0}            & \textbf{0}                     & \textbf{0}                      \\
    depot (22)      & \textbf{4}                           & 2                                      & 2                              & 3                               & 0                      & 0                     & 0                              & \textbf{1}                      \\
    driverlog (20)  & \textbf{10}                          & 8                                      & 6                              & 8                               & 1                      & 1                     & 1                              & \textbf{2}                      \\
    freecell (80)   & \textbf{15}                          & 13                                     & 5                              & 5                               & \textbf{0}             & \textbf{0}            & \textbf{0}                     & \textbf{0}                      \\
    grid (5)        & \textbf{2}                           & \textbf{2}                             & 1                              & 1                               & \textbf{1}             & 0                     & 0                              & \textbf{1}                      \\
    gripper (20)    & \textbf{3}                           & 2                                      & 2                              & \textbf{3}                      & 1                      & 1                     & 0                              & \textbf{2}                      \\
    hiking (20)     & 4                                    & 3                                      & 1                              & \textbf{7}                      & 0                      & 0                     & 0                              & \textbf{1}                      \\
    logistics (63)  & 9                                    & 6                                      & 4                              & \textbf{13}                     & 1                      & 1                     & 0                              & \textbf{3}                      \\
    miconic (150)   & \textbf{39}                          & 35                                     & 31                             & \textbf{39}                     & 14                     & 13                    & 10                             & \textbf{24}                     \\
    movie (30)      & 2                                    & 2                                      & 0                              & \textbf{30}                     & 0                      & 0                     & 0                              & \textbf{30}                     \\
    mprime (35)     & 22                                   & 20                                     & 22                             & \textbf{23}                     & \textbf{12}            & 7                     & 2                              & 9                               \\
    mystery (19)    & \textbf{16}                          & 14                                     & 14                             & 15                              & \textbf{11}            & 8                     & 7                              & 9                               \\
    nomystery (20)  & \textbf{14}                          & 13                                     & 8                              & 8                               & \textbf{5}             & 2                     & 1                              & 4                               \\
    organic (16)    & \textbf{7}                           & \textbf{7}                             & 0                              & 0                               & \textbf{7}             & \textbf{7}            & 0                              & 0                               \\
    parking (40)    & \textbf{3}                           & 1                                      & 0                              & 0                               & \textbf{0}             & \textbf{0}            & \textbf{0}                     & \textbf{0}                      \\
    pipes-nt (46)   & \textbf{16}                          & 11                                     & 10                             & 12                              & 2                      & 1                     & 1                              & \textbf{3}                      \\
    pipe-t (45)     & \textbf{9}                           & 7                                      & 5                              & 8                               & \textbf{2}             & 1                     & 1                              & \textbf{2}                      \\
    psr-small (50)  & 46                                   & 44                                     & 41                             & \textbf{48}                     & 14                     & 14                    & 8                              & \textbf{24}                     \\
    quantum (20)    & \textbf{10}                          & 8                                      & 9                              & 9                               & \textbf{2}             & 1                     & 1                              & \textbf{2}                      \\
    rovers (40)     & \textbf{4}                           & \textbf{4}                             & \textbf{4}                     & \textbf{4}                      & 0                      & 0                     & 0                              & \textbf{4}                      \\
    satellite (36)  & 5                                    & 5                                      & 5                              & \textbf{6}                      & \textbf{1}             & \textbf{1}            & 0                              & \textbf{1}                      \\
    snake (20)      & \textbf{6}                           & 5                                      & 1                              & 1                               & \textbf{2}             & 0                     & 0                              & 0                               \\
    storage (29)    & \textbf{16}                          & 15                                     & 12                             & 12                              & \textbf{7}             & 6                     & 5                              & \textbf{7}                      \\
    termes (20)     & 5                                    & \textbf{6}                             & 2                              & 2                               & \textbf{0}             & \textbf{0}            & \textbf{0}                     & \textbf{0}                      \\
    tidybot (40)    & \textbf{20}                          & 10                                     & 4                              & 5                               & \textbf{1}             & \textbf{1}            & \textbf{1}                     & \textbf{1}                      \\
    tpp (30)        & \textbf{5}                           & 4                                      & 4                              & \textbf{5}                      & 3                      & 3                     & 3                              & \textbf{4}                      \\
    visitall (40)   & 12                                   & \textbf{16}                            & \textbf{16}                    & \textbf{16}                     & 5                      & 5                     & 5                              & \textbf{6}                      \\
    zenotravel (20) & \textbf{9}                           & 8                                      & 8                              & 8                               & \textbf{4}             & 3                     & 2                              & \textbf{4}                      \\
    \midrule
    Sum (1094)      & \textbf{351}                         & 309                                    & 253                            & 335                             & 112                    & 91                    & 61                             & \textbf{170}                    \\
\end{tabular}

    \caption{
        Coverage per domain, i.e., number of tasks per domain where the number of plans within a factor $1.0$ or $1.5$ of a cost bound was found by \toolname{K^*}, \toolname{SymK}, and our SAT-based approaches, \toolname{Count} and \toolname{Enum}.
        \toolname{Count} only counts plans, while \toolname{Enum} outputs each plan. %
      }
    \label{tab:domain_coverage}
\end{table}

\subsection{Domain-Wise Performance}
Table \ref{tab:domain_coverage} shows a domain-wise comparison of the different approaches for the two extreme bounds in our experiments, $1.0$ and $1.5$.
For both bounds, the performance differs a lot depending on the domain.
Our SAT-based approach performs particularly well in the blocksworld and psr-small domains in both cases.
In blocksworld, the largest task that we could still solve had $1.5 \cdot 10^{9}$ plans, while in psr-small the largest solved task had $8.9 \cdot 10^{12}$.
In contrast, \toolname{K^*} could only count up to a 10 million plans in these domains.

The SAT-based approach is less effective in other domains.
One reason %
is that %
they are less specialized than heuristic and symbolic search approaches to optimal planning.
Among other factors, the sequential encoding is not concise enough for some tasks and bounds (e.g., airport), or the grounding algorithm of \toolname{Madagascar} is inferior to those of other planners built on top of the \toolname{Fast Downward} grounder \cite{helmert-jair2006,helmert-aij2009}, making it impossible to ground certain tasks (e.g., organic-synthesis). It would be interesting to evaluate how other encodings perform \cite{rintanen-aij2012}, but that brings the additional problem of losing the one-to-one correspondence between plans and SAT models.

For %
$1.5$, counting is more feasible than enumeration in many domains:
as the number of plans increases, enumeration becomes less practical.
Counting %
works %
for many reasoning tasks, e.g., those based on conditional probabilities.

\subsection{Beyond Counting}
As illustrated above, our \toolname{Planalyst} tool effectively counts plans by compiling  into a d-DNNF and performing a counting query.
This method can not only answer conditional probability questions, such as the quantity of an operator in plans, but also addresses other reasoning questions more directly and efficiently through d-DNNF queries using \toolname{ddnnife} \cite{sundermann-et-al-tsem2024}.
Consider %
reasoning questions about the plan space of a given planning task, while respecting a cost bound.
Given the d-DNNF representing the plan space, questions about brave and cautious operators can be answered directly, even without traversing the entire d-DNNF, when the number of plans is known \cite{sundermann-et-al-tsem2024}. 
This can be achieved by traversing the literal nodes of the d-DNNF and collecting the backbone variables, i.e., the variables that are always true (core) or false (dead).
In addition, given the d-DNNF, it is possible to uniformly sample plans without enumerating the full set by d-DNNF traversing with \toolname{ddnnife}.
This allows to address planning biases when selecting plans \cite{paredes-et-al-icapswsrddps2024,frank-et-al-icapswsrddps2024} and thus collect unbiased training data for different learning approaches \cite{shen-et-al-icaps2020,areces-icaps2023wskeps,chen-et-al-aaai2024,bachor-behnke-aaai2024}.
We omit empirical results for these queries, as their overhead is negligible once the d-DNNF is constructed.
Our experiments with the \toolname{Count} configuration of \toolname{Planalyst} have shown that this construction is feasible for many planning tasks.

\section{Conclusion and Future Work}

We %
count plans and reason in the solution space, which
is orthogonal to previous works in
planning~\cite{katz-et-al-icaps2018,speck-et-al-aaai2020,katz-sohrabi-aaai2020}.
Moreover, we reason about the plan space in the form of queries and
introduce faceted reasoning to planning allowing for questions on the
significance of operators. Although faceted reasoning is
computationally hard (\NP-c), it is, under standard theoretical
assumptions, significantly more efficient than counting the number of
plans (\sharpP-c).
Finally, we present our new reasoning tool, \toolname{Planalyst},
which can count the number of plans assuming fixed given length. It
also supports different plan space queries. In general,
\toolname{Planalyst} is competitive with state-of-the-art top-k
planners %
and outperforms all other methods
when the plan space is too large,~i.e., more than 10 million plans.

In the future, we will integrate \toolname{Planalyst} into other pipelines,
such as goal recognition \cite{mirsky-et-al-book2021}, grounding
via learning \cite{gnad-et-al-aaai2019}, and task rewriting
\cite{areces-et-al-icaps2014,elahi-rintanen-aaai2024}. We believe counting
and facet reasoning are useful for guidance in these areas.
Interesting topics for considerations could be to deal with inconsistencies~\cite{Ulbricht19}
and certifying results~\cite{AlvianoDodaroFichte19,FichteHecherRoland22}
as well as explaining reasoning behind decisions~\cite{CabalarFandinnoMuniz20}.
We will study how our framework extends to other encodings, such
as parallel operator encodings \cite{rintanen-aij2012} or lifted encodings
\cite{hoeller-behnke-icaps2022}.

\section{Acknowledgements}
Authors are ordered in reverse alphabetical order.
David Speck was funded by the Swiss National Science Foundation (SNSF) as part of the project ``Unifying the Theory and Algorithms of Factored State-Space  Search'' (UTA).
Hecher was supported by %
the Austrian Science Fund (FWF), grants J 4656 and P 32830, 
the Society for Research Funding in Lower Austria (GFF, Gesellschaft für Forschungsf\"orderung N\"O), grant ExzF-0004, 
as well as the Vienna Science and Technology Fund (WWTF), grant ICT19-065.
The work has been carried out while Hecher visited the Simons
Institute at UC Berkeley.
Fichte was funded by ELLIIT funded by the Swedish government.

{%

}

\cleardoublepage
\appendix
\section{Appendix}
\subsection{Additional Preliminaries}

\paragraph{Propositional Logic}
We define propositional (Boolean) formulas and their evaluation in the
usual way~\cite{kleine-lettmann-1999,RobinsonVoronkov01}.
\emph{Literals} are propositional variables or their negations.
For a propositional formula~$F$, we denote by $\at(F)$ the set of
variables that occur in formula~$F$.
Logical operators~$\wedge$, $\vee$, $\neg$ $\rightarrow$,
$\leftrightarrow$ are used in the usual meaning.
A \emph{term} is a conjunction ($\wedge$) of literals and a clause is
a \emph{disjunction} ($\vee$) of literals.
Formula~$F$ is in \emph{conjunctive normal form (CNF)} if $F$ is a
conjunction of clauses.
We abbreviate by $\Mod(F)$ the set of all models of~$F$ and the number
of models by $\#(F)\eqdef \Card{\Mod(F)}$.

\paragraph{Knowledge Compilation} %
Knowledge compilation is a sub-area of automated reasoning and
artificial intelligence where one transforms propositional formulas
into certain normal forms on which reasoning operations such as
counting can be carried out in polynomial
time~\cite{darwiche-marquis-jair2002,DarwicheMarquis24}.
In our case, the general outline for a given planning task~$\Pi$ is as
follows:
\newcommand{\NNF}{\ensuremath{F_{\text{NF}}}}
\begin{enumerate}
\item We construct the propositional CNF formula~$\fplan[\Pi]$.
\item Then, we compile~$\fplan[\Pi]$ in a computationally expensive
  step into a formula~$\NNF$ in a normal form, so-called
  d-DNNF by existing knowledge compilers.
\item Finally, on the formula~$\NNF$ counting (and other operations)
  can be done in polynomial time in the size of~$\NNF$.  We can even
  count under a set~$L$ of propositional assumptions by the technique
  known as conditioning.
\end{enumerate}
In more detail:
Let $F$ be a (propositional) formula, $F$ is in \emph{NNF
  (negation normal form)} if negations ($\neg$) occur only directly in
front of variables and the only other operators are conjunction
($\wedge$) and disjunction
($\vee$)\longversion{~\cite{RobinsonVoronkov01}}.  NNFs can be
represented in terms of {\em rooted directed acyclic graphs} (DAGs)
where each leaf node is labeled with a literal, and each internal node
is labeled with either a conjunction (\emph{$\wedge$-node}) or a
disjunction (\emph{$\vee$-node}).
The \emph{size of an NNF}~$F$, denoted by~$|F|$, is given by the
number of edges in its DAG.  Formula~$F$ is in \emph{DNNF}, if it is
in NNF and it satisfies the \emph{decomposability} property, that is,
for any distinct sub-formulas $F_i, F_j$ in a conjunction
$F=F_1\land\dots \land F_n$ with $i\neq j$, we have
$\at(F_i)\cap\at(F_j)=\emptyset$\longversion{~\cite{Darwiche04a}}.
Formula $F$ is in \emph{d-DNNF}, if it is in DNNF and it satisfies the
\emph{decision} property, that is, disjunctions are of the form
$F=(x\land F_1)\lor(\neg x\land F_2)$.  Note that $x$ does not occur
in $F_1$ and $F_2$ due to decomposability. $F_1$ and $F_2$ may be
conjunctions.  Formula $F$ is in \emph{sd-DNNF}, if all disjunctions
in $F$ are smooth, meaning for $F=F_1\lor F_2$ we have
$\at(F_1) = \at(F_2)$.
Determinism and smoothness permit traversal operators on sd-DNNFs to
count models of~$F$ in linear time
in~$|F|$\longversion{~\cite{darwiche2001tractable}}. The traversal
takes place on the so-called counting graph of an sd-DNNF.  The {\em
  counting graph}~$\mathbb{G}(F)$ is the DAG of~$F$ where each
node~$N$ is additionally labeled by $\val(N) \eqdef 1$, if $N$
consists of a literal; labeled by
$\val(N) \eqdef \Sigma_{i} \val(N_i)$, if $N$ is an $\vee$-node with
children~$N_i$; labeled by $\val(N) \eqdef \Pi_{i} \val(N_i)$, if $N$
is an $\wedge$-node.  By~$\val(\mathbb{G}(F))$ we refer to~$\val(N)$
for the root~$N$ of~$\mathbb{G}(F)$.  Function $\val$ can be
constructed by traversing $\mathbb{G}(F)$ in post-order in polynomial
time.
It is well-known that $\val(\mathbb{G}[F])$ equals the model count of~$F$.
For a set~$L$ of literals, counting of
$F^L \eqdef F \wedge \bigwedge_{\ell \in L}\ell$ can be carried out by
\emph{conditioning} of~$F$ on~$L$~\cite{darwiche-ijcai1999}.
Therefore, the function~$\val$ on the counting graph is modified by
setting $\val(N) = 0$, if~$N$ consists of~$\ell$
and~$\neg \ell \in L$.  This corresponds to replacing each
literal~$\ell$ of the NNF~$F$ by constant~$\bot$ or $\top$,
respectively.
Similarly, we can enumerate models or compute brave/cautious
operators.

\section{Omitted Proofs}
\obsqualitative*
\begin{proof}[Proof (Sketch)]
  (Membership): Let
  $\Pi$ %
  be a planning task, $o \in \operators$ an operator, and
  $\ell$ an integer.
  We can simply conjoin the
  formula~$\big((\bigvee_{i \in [\ell]} o^\ell)\rightarrow o\big)$ to
  formula~$\fplan[\Pi]$, which ensures that variable~$o$ is true if
  operator~$o$ occurs in a plan in position~$i$ of a plan.
  (Remember that $\ell$ is guaranteed to be polynomially bounded, so $\fplan[\Pi]$ is also polynomial.)
  For brave operators, we
  conjoin~$o$ and ask whether the resulting formula is satisfiable,
  which gives $\NP$-membership.
  For cautious operators, we conjoin~$\neg
  o$, ask for satisfiability, and swap answers, which immediately
  yields $\co\NP$-membership.
  (Hardness): We can vacously extend the existing reduction \cite[The
  3.5]{bylander-aij1994} and ask for brave (SAT) and cautious (UNSAT).
\end{proof}

\thmprob*
\begin{proof}

	(Hardness): We reduce from the problem of deciding whether the number of accepting
paths of a non-deterministic Turing machine equals its number of rejecting paths, see, e.g.,~\cite{fenner-et-al-eccc1999}.
	Indeed, we can encode Turing machine acceptance into a propositional formula using the Cook-Levin reduction~\cite{cook-stoc1971} which is parsimonious,
	i.e., the number of satisfying assignments precisely preserve the number of accepting paths~\cite[Lemma 3.2]{valiant-tcs1979}.
	Analogously, one can encode the number of rejecting paths in a propositional formula, by inversion and De Morgan's law (or Tseitin transformation).
	Consequently, we can also construct a formula~$F$ having $\# acc paths$ satisfying assignments if $acc$ is set to true and $\# rej paths$
	satisfying assignments in case $acc$ is set to false. %
	Observe that we can solve the Turing machine counting problem by asking whether $\nicefrac{\# acc paths}{\# acc paths + \# rej paths} = 0.5$, %
	which boils down to asking whether $\nicefrac{\#(F \cup \{acc\})}{\#(F)} = 0.5$.
	Indeed, this can be solved via probabilistic reasoning,
	by parsimoniously reducing $F$ into a planning problem~\cite{speck-et-al-aaai2020} and asking for the probability of operator $acc$. %

	(Membership): We reduce
        $\nicefrac{\Card{\AP_\ell(\Pi,Q)}}{\max(1,\Card{\AP_\ell(\Pi)})}$ %
        ${=\,}p {\,=\,}\nicefrac{n}{d}$ for a given query~$Q$ to asking whether
        $d\Card{\AP_\ell(\Pi,Q)}=n\Card{\AP_\ell(\Pi)}$. This clearly
        works in
        $\CE$ by parsimoniously reducing both planning counting tasks
        to a propositional formula and checking for equality of their
        number of satisfying assignments.
\end{proof}

\thecomplexityfacetdec*
\begin{proof}
  (Membership): As in Lemma~\ref{obs:qualitative}, we encode
  task~$\Pi$, %
  operator $o \in \operators$, and integer
  $\ell$ into a propositional formula $\fplan[\Pi]$.
  Then, we let $F\eqdef \big((\bigvee_{i \in [\ell]} o^\ell)\rightarrow o\big) \wedge \fplan[\Pi]$, which ensures that variable~$o$ is true if
  operator~$o$ occurs in a plan in position~$i$ of a plan.
  Then, we construct a fresh formula~$F'$ where every variable~$v$ in $F$ is renamed to a fresh variable~$v'$.
  Finally, formula~$\ft{facet}\eqdef F \wedge o \wedge F' \wedge \neg o'$ is satisfiable if and only if $o\in \FA(\Pi)$.

  (Hardness): Take any propositional formula~$F$. We ensure that $F$ is not a tautology,
  by adding to~$F$ the trivial clause $\neg v$ over fresh variable~$v$, which results in formula~$F'$.
  Then, we parsimoniously reduce a propositional formula $F'$ into a planning problem~$\Pi = \langle \vars, \operators, \init, \goal \rangle$~\cite{speck-et-al-aaai2020}. %
  In particular, this translation ensures a one-to-one correspondence between the satisfying assignments of~$F'$ and the plans of~$\Pi$.
  Observe that we can slightly adapt this planning instance, resulting in~$\Pi'$, where we add a single goal operator $o$
  that is applicable if precondition $\goal$ holds.
  Then, $F$ is satisfiable if and only if $o \in \FA(\Pi')$.
\end{proof}

\thmatlestfacet*
\begin{proof}
(Membership): Follows from the membership of the proof of Theorem~\ref{the:complexity:facet_dec}.
Indeed, we take the constructed formula~$\ft{facet}$ and conjunctively cojoin it~$k$ times (over fresh variables),
resulting in a formula~$\ft{facet}^{\geq k}=\ft{facet}^1\wedge \ldots \wedge \ft{facet}^k$.
Assume an arbitrary, but fixed, ordering~$\prec$ among the variables in~$\ft{facet}$, which we naturally extend over any copy formula~$\ft{facet}^i$.
Then, for~$2\leq i\leq k$ we encode that the satisfying assignment over copy $\ft{facet}^i$ is $\prec$-larger than the satisfying assignment over $\ft{facet}^{i-1}$. This is the case if there is variable~$v_i$ in $\ft{facet}^i$ that is set to true, but $v_{i-1}$ in $\ft{facet}^{i-1}$ is set to false, such that all $\prec$-larger variables in $\ft{facet}^{i-1}$ are set to false.

(Hardness): We reduce from an arbitrary propositional formula~$F$ to a planning task~$\Pi_F$.
We apply the same approach as in Theorem~\ref{the:complexity:facet_dec}, but
we need to make every facet candidate into a facet, which then allows us to ask for $\geq |\operators_F|$ (all) facets.
\end{proof}

\corcomplexityatmostfacet*
\begin{proof}
This is the co-Problem of $\Card{\FA(\Pi)} \geq k+1$; therefore the result follows
directly from Lemma~\ref{thm:atlestfacet}.
\end{proof}

\thmexactlyfacet*
\begin{proof}
(Membership): Follows from memberships of Lemma~\ref{thm:atlestfacet} and Corollary~\ref{cor:complexity:atmostfacet}.

(Hardness): Follows from hardness of Lemma~\ref{thm:atlestfacet} and Corollary~\ref{cor:complexity:atmostfacet}. Indeed, we can reduce from arbitrary propositional formulas~$F_{\text{sat}}, F_{\text{unsat}}$ to decide whether
$F_{\text{sat}}$ is satisfiable and $F_{\text{unsat}}$ is unsatisfiabile, if and only if for planning task~$\Pi_{F_\text{sat}}$ all candidate facets are facets ($\geq |\operators_{F_\text{sat}}|$), but not all candidate facets ($\leq |\operators_{F_\text{unsat}}|-1$) for $F$ are facets for~$\Pi_{F_\text{unsat}}$.
\end{proof}

\end{document}